\documentclass[letterpaper, 10 pt, conference]{ieeeconf}  % Comment this line out if you need a4paper

\IEEEoverridecommandlockouts                              % This command is only needed if 
\usepackage{cite}
\usepackage{url}
\usepackage{graphicx}
\usepackage{amsfonts}
\usepackage{amsmath}
\usepackage{MnSymbol}
\usepackage{subcaption}
\makeatletter
\def\th@plain{%
	\thm@notefont{}% same as heading font
	\normalfont % body font
}
\def\th@definition{%
	\thm@notefont{}% same as heading font
	\normalfont % body font
}
\makeatother
\usepackage{tikz,pgfplots}

\newcommand{\sig}[1]{\textsf{{#1}}}

\newtheorem{lemma}{Lemma}
\newtheorem{defn}{Definition}

\title{\LARGE \bf
Towards Robust Direct Perception Networks for Automated Driving
}

\author{Chih-Hong Cheng% <-this % stops a space
\thanks{Chih-Hong Cheng is with Corporate R\&D, DENSO AUTOMOTIVE Deutschland GmbH.
        {\tt\small c.cheng@denso-auto.de}}%
}

\begin{document}

\maketitle
\thispagestyle{empty}
\pagestyle{empty}

%%%%%%%%%%%%%%%%%%%%%%%%%%%%%%%%%%%%%%%%%%%%%%%%%%%%%%%%%%%%%%%%%%%%%%%%%%%%%%%%
\begin{abstract}

We consider the problem of engineering \emph{robust} direct perception  neural networks with  output being regression. Such networks take high dimensional input image data, and they produce affordances such as the curvature of the upcoming road segment or the distance to the front vehicle. Our proposal starts by allowing a neural network prediction to deviate from the label with  tolerance~$\Delta$. The source of tolerance can be either contractual or from limiting factors where two entities may label the same data with slightly different numerical values. The tolerance motivates the use of a non-standard loss function where the loss is set to~$0$ so long as the prediction-to-label distance is less than~$\Delta$. We further extend the loss function and define a new \emph{provably robust} criterion that is parametric to the allowed output tolerance~$\Delta$, the layer index~$\tilde{l}$ where perturbation is considered, and the maximum perturbation amount~$\kappa$. During training, the robust loss is computed by first propagating symbolic errors from the $\tilde{l}$-th layer (with quantity bounded by $\kappa$) to the output layer, followed by computing the overflow between the error bounds and the allowed tolerance. The overall concept is experimented in engineering a direct perception neural network for understanding the central position of the ego-lane in pixel coordinates.

\end{abstract}

%%%%%%%%%%%%%%%%%%%%%%%%%%%%%%%%%%%%%%%%%%%%%%%%%%%%%%%%%%%%%%%%%%%%%%%%%%%%%%%%
\section{Introduction}\label{sec:intro}

Deep neural networks (DNNs) are increasingly used in the automotive industry for realizing perception functions in automated driving. The safety-critical nature of automated driving requires the deployed DNNs to be highly dependable. Among many dependability attributes we consider the \emph{robustness} criterion, which intuitively requires a neural network to produce  similar output values under similar inputs. It is known that DNNs trained under standard approaches can be difficult to exhibit robustness. For example, by imposing carefully crafted tiny noise on an input data point, the newly generated data point may enable a DNN to produce results that completely deviate from the originally expected output. 

In this paper, we study the robustness problem for \emph{direct perception} networks in automated driving. The concept of direct perception neural networks refers to learning affordances (high-level sensory features)~\cite{pomerleau1989alvinn,chen2015deepdriving,sauer2018conditional} such as distance to lane markings or distance to the front vehicles, directly from high-dimensional sensory inputs. 
In contrast to classification where the output criterion for robustness is merely the \emph{sameness} of the output label for input data under perturbation, direct perception commonly uses \emph{regression} output. For practical systems, it can be unrealistic to assume that trained neural networks can produce output regression that perfectly matches the numerical values as specified in the labels. 

Towards this issue, our proposal is to define \emph{tolerance}~$\Delta$ that explicitly regulates the allowed output deviation from labels. Pragmatically, the source of tolerance arises from two aspects, namely (i) the quality contract between car makers and their suppliers, or (ii) the inherent uncertainty in the manual labelling process\footnote{Our very preliminary experiments demonstrated that, when trying to label the center of the ego lane on the same image having~$1280$ pixels in width, a labelling deviation of~$10$ pixels for two consecutive trials is very common, especially when the labelling decision needs to be conducted within a short period of time.}. We subsequently define a loss function that integrates the tolerance - the prediction error is set to be~$0$ so long as the prediction falls within the tolerance bound. Thus the training intuitively emphasizes reducing the \emph{worst case} (i.e., to bring the prediction back to the tolerance). This is in contract to the use of standard loss functions (such as mean-squared-error) where the goal is to bring every prediction to be close to the label. 

As robustness requires that input data being perturbed should produce  results similar to input data without perturbation, $\Delta$ can naturally be overloaded to define the ``sameness" of the regression output under perturbation. Based on this concept, we further propose a new  criterion for \emph{provable robustness}~\cite{kolter2017provable,sinha2017certifiable,wang2018mixtrain,raghunathan2018certified,wong2018scaling,tsuzuku2018lipschitz,salman2019provably} tailored for regression, which is parametric to the allowed output tolerance~$\Delta$, the layer index~$\tilde{l}$ where perturbation is considered, and the maximum perturbation amount~$\kappa$. The robust criterion requires that for any data point in the training set, by applying any feature-level perturbation on the $l$-th layer with quantity less than~$\kappa$, the computation of the DNN only leads to slight output deviation (bounded by~$\Delta$) from the associated ground truth. Importantly, the introduction of parameter~$\tilde{l}$ overcomes scalability and precision issues, while it also implicitly provides capabilities to capture global input transformations (cf. Section~\ref{sec.related} for a detailed comparison to existing work). By carefully defining the loss function as the \emph{interval overflow} between (i) the computed error bounds due to perturbation and (ii) the allowed tolerance interval, the loss can be efficiently computed by summing the overflow of two end-points in the propagated symbolic interval.

To evaluate our proposed approach, we have trained a direct perception network with labels created from publicly accessible datasets. The network takes input from road images and produces affordances such as the central position of the ego-lane in pixel coordinates. The positive result of our preliminary experiment demonstrates the potential for further applying the technology in other automated driving tasks that use DNNs.

\vspace{2mm}
\noindent{\textbf{(Structure of the Paper)}} 
The rest of the paper is structured as follows. Section~\ref{sec.preliminaries} starts with basic formulations of neural networks and describes the tolerance concept. It subsequently details how the error between predictions and labels, while considering tolerance, can be implemented with GPU support. Section~\ref{sec.robust.training} extends the concept of tolerance for provable robustness by considering feature-level perturbation. Section~\ref{sec.evaluation} details our initial experiment in a highway vision-based perception system. Finally, we outline related work in Section~\ref{sec.related} and conclude the paper in Section~\ref{sec.concluding.remarks} with future directions.

\section{Neural Network and Tolerance}~\label{sec.preliminaries}

A \emph{neural network} $\mathcal{N}$ is comprised of~$L$ layers
where operationally, the $l$-th layer for $l\in\{1,\dots,L\}$ of the network is a function $g^{(l)}: \mathbb{R}^{d_{l-1}} \rightarrow \mathbb{R}^{d_{l}}$, with $d_{l}$ being the dimension of layer~$l$.  
Given an input data point $\sig{in} \in \mathbb{R}^{d_{0}}$, the output of the $l$-th layer of the neural network $f^{(l)}$ is given by the functional composition of the $l$-th layer and previous layers $f^{(l)}(\sig{in}) := \circ_{i=1}^{(l)} g^{(i)}(\sig{in})  = g^{(l)}(\ldots g^{(2)}(g^{(1)}(\sig{in})))$. $f^{(L)}(\sig{in})$ is the \emph{prediction} of the network under input data point~$\sig{in}$. Throughout this paper, we use subscripts to extract an element in a vector, e.g., use $\sig{in}_j$ to denote the $j$-th value of~$\sig{in}$.

Given a neural network $\mathcal{N}$  following above definitions, let $\mathcal{D}_{train} := \{ (\sig{in}, \sig{lb})\}$ be the \emph{training data set}, with each data point $\sig{in} \in \mathbb{R}^{d_{0}}$ having its associated label~$\sig{lb} \in \mathbb{R}^{d_{L}}$. Let $(\Delta_1,\dots ,\Delta_{d_{L}})$, where  $\forall j \in \{1, \ldots, d_{L}\}: \Delta_j \geq 0$, be the output \emph{tolerance}. We integrate tolerance to define the error between a prediction $f^{(L)}(\sig{in})$ of the neural network and the label~$\sig{lb}$. Precisely, for output index $j \in \{1, \ldots, d_{L}\}$,

\begin{equation} e^{\Delta}_{j}(f^{(L)}_j(\sig{in}), \sig{lb}_j) := \begin{cases}
0 \;\;\;\;\;\;\;\;\;\;\;\;\;\;\;\;\;\;\;\;\;\;  \text{if} \; |f^{(L)}_j(\sig{in}) - \sig{lb}_j| \leq \Delta_j  \\
\sig{min}( |f^{(L)}_j(\sig{in}) - (\sig{lb}_j - \Delta_j)|, \\\;\;\;\; \;\;\;\; |f^{(L)}_j(\sig{in}) - (\sig{lb}_j + \Delta_j)|) \;\; \text{otherwise}
\end{cases} 
\end{equation}

\begin{figure}
	\includegraphics[width=\columnwidth]{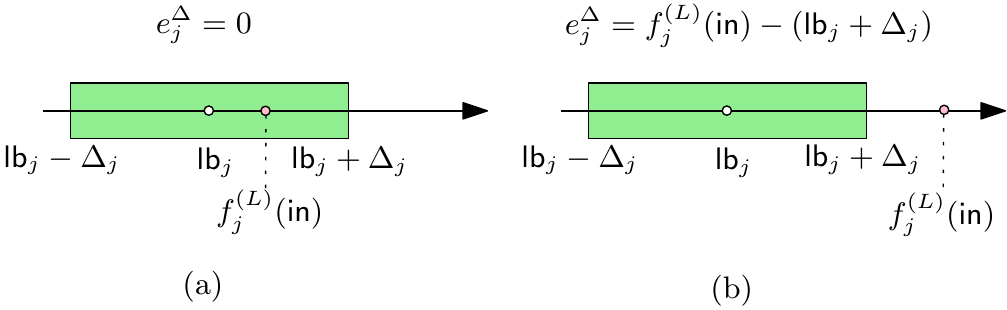}
	\caption{Illustrating the error defined by $\Delta_j$.}
	\label{fig:loss}
\end{figure}

\vspace{2mm}
Figure~\ref{fig:loss} illustrates the intuition of such an error definition. We consider a prediction to be correct (i.e., error to be~$0$) when the prediction is within interval $[\sig{lb}_j - \Delta_j, \sig{lb}_j + \Delta_j]$ (Figure~\ref{fig:loss}-a). Otherwise, the error is the distance to the boundary of the interval (Figure~\ref{fig:loss}-b). Finally, we define the \emph{in-sample error} to be the sum of squared error for each output dimension, for each data point in the training set.

\vspace{1mm}
\begin{defn}[Interval tolerance loss]\label{def.interval.error} 
Define the error in the training set (i.e., in-sample error) to be
	$
	E^{\Delta}_{train} := \frac{1}{N}\sum_{(\sig{in}, \sig{lb}) \in \mathcal{D}_{train}} (E(\sig{in}, \sig{lb}))
	$, where  $N = |\mathcal{D}_{train}|$ and $E(\sig{in}, \sig{lb}) = \sum^{d_{L}}_{j=1} (e^{\Delta}_{j}(f^{(L)}_j(\sig{in}), \sig{lb}_j))^{2}$ .
\end{defn}
\vspace{1mm}

One may observe that the loss function defined above is designed as an extension of the mean-squared-error (MSE) loss function. 

\vspace{1mm}
\begin{lemma}\label{lemma.equi.to.mse} When $\forall j \in \{1, \ldots, L\}:\Delta_j = 0$, $E^{\Delta}_{train}$ is equal to the mean-squared-error loss function $\frac{1}{N}\sum_{(\sig{in}, \sig{lb}) \in \mathcal{D}_{train}} ||f^{(L)}(\sig{in}) - \sig{lb}\,||^2$.  
\end{lemma}

\begin{proof}
	By setting $\Delta_j = 0$, $e^{\Delta}_{j}(f^{(L)}_j(\sig{in}), \sig{lb}_j)$ can be simplified to $|f^{(L)}_j(\sig{in}) - \sig{lb}_j|$. Thus, in Definition~\ref{def.interval.error}, $E(\sig{in}, \sig{lb})$ is simplified to $\sum^{d_{L}}_{j=1} |f^{(L)}_j(\sig{in}) - \sig{lb}_j|^{2}$, which is equivalent to computing the square of the L2-norm $||f^{(L)}(\sig{in}) - \sig{lb}\,||^2$.
\end{proof}

\vspace{2mm}
\noindent \textbf{(Implementing the loss function with GPU support)} For commonly seen machine learning infrastructures such as TensorFlow\footnote{Google TensorFlow: \url{http://www.tensorflow.org}} or PyTorch\footnote{Facebook PyTorch: \url{https://www.pytorch.org}},  for training a neural network that uses standard layers such as ReLU~\cite{nair2010rectified}, ELU~\cite{clevert2015fast}, Leaky ReLU~\cite{maas2013rectifier} as well as convolution, one only needs to manually implement the customized loss function, while back propagation capabilities for parameter updates are automatically  created by the  infrastructure. In the following, we demonstrate a rewriting of $e^{\Delta}_j()$ such that it uses built-in primitives supported by TensorFlow. Such a rewriting makes it possible for the training to utilize GPU parallelization.

\vspace{1mm}
\begin{lemma}[Error function using GPU function primitives]\label{lemma.gpu.primitive}
	Let $\sig{clip}_{\geq 0}(x)$ be a function that returns~$x$ if $x \geq 0$; otherwise it returns~$0$. 
	Define $\hat{e}^{\Delta}_{j}(f^{(L)}_j(\sig{in}), \sig{lb}_j)$ to be
	$ \sig{max}(\sig{clip}_{\geq 0}((\sig{lb}_j - \Delta_j) - f^{(L)}_j(\sig{in})),   \sig{clip}_{\geq 0}(f^{(L)}_j(\sig{in}) - (\sig{lb}_j + \Delta_j)))$. Then $e^{\Delta}_{j}(f^{(L)}_j(\sig{in}), \sig{lb}_j) = \hat{e}^{\Delta}_{j}(f^{(L)}_j(\sig{in}), \sig{lb}_j)$.  
\end{lemma}
\vspace{2mm}  

\begin{proof}(Sketch)
	It can be intriguing to reason that $\hat{e}^{\Delta}_{j}()$ and $e^{\Delta}_{j}()$  are equivalent functions, i.e., the if-then-else statement in~$e^{\Delta}_{j}()$ is implicitly implemented using the $\sig{clip}_{\geq 0}$ primitive\footnote{The $\sig{clip}_{\geq 0}()$ function is implemented in Google TensorFlow using $\texttt{tf.keras.backend.clip}$.} in~$\hat{e}^{\Delta}_{j}()$. To assist understanding, Figure~\ref{fig:equivalence} provides a simplified proof by enumerating all three possible cases regarding  the relative position of $f^{(L)}_j(\sig{in})$ and the tolerance interval, together with their intermediate computations. Diligent readers can easily swap the constants in Figure~\ref{fig:equivalence} and establish a formal correctness proof. 
\end{proof}

\begin{figure}
	\includegraphics[width=\columnwidth]{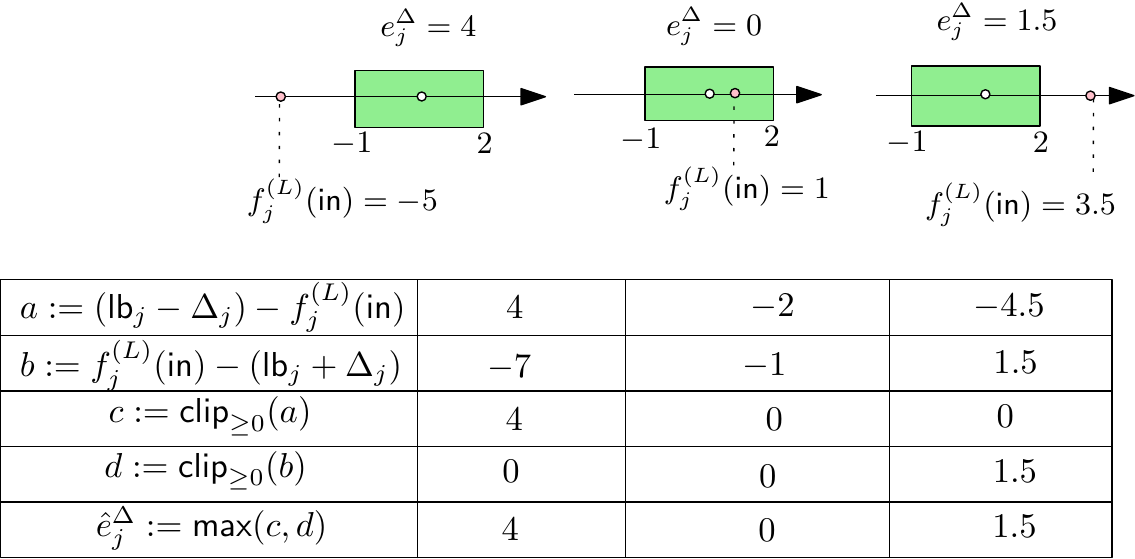}
	\caption{Illustrating the output equivalence  between functions $e^{\Delta}_j()$ and $\hat{e}^{\Delta}_j()$ using examples.}
	\label{fig:equivalence}
\end{figure}  

\section{Provably Robust Training}~\label{sec.robust.training}

This section starts  by outlining the concept of feature-level perturbation, followed by defining symbolic loss. It then defines provably robust training and how training can be made efficient with GPU support. For simplifying notations, in this section let~$v \oplus c$ be an operation that (1) if $c$ is a scalar, adds $c$ to every dimension of a vector~$v$, or (2) if $c$ is a vector, perform element-wise addition. 

\vspace{1mm}
\begin{defn}[Output bound under $(\tilde{l}, \kappa)$-perturbation]\label{def.perturbation}
	Given a neural network $\mathcal{N}$ and an input data point~$\sig{in}$, let 
		$[L, U]$ be the output bound
subject to	$(\tilde{l}, \kappa)$-perturbation. For each output dimension $j \in \{1, \ldots, d_{L}\}$, 	$[L_j, U_j]$, the $j$-th output bound  of the neural network, satisfies the following condition: 
	If $o_j = g^{(L)}_j(g^{(L-1)}(\ldots(g^{(\tilde{l})}(\sig{fv}))))$ where~$\sig{fv} \in [f^{(\tilde{l}-1)}(\sig{in}) \oplus -\kappa, \; f^{(\tilde{l}-1)}(\sig{in}) \oplus \kappa]$, then $o_j \in [L_j, U_j]$. 
	
\end{defn}

\vspace{2mm}

Definition~\ref{def.perturbation} can be understood operationally: first compute $f^{(\tilde{l}-1)}(\sig{in})$ which is the feature vector of~$\sig{in}$ at layer~$\tilde{l}$. Subsequently, try to perturb~$f^{(\tilde{l}-1)}(\sig{in})$ with some noise bounded by $[-\kappa, \kappa]$ in each dimension, in order to create a perturbed feature vector~$\sig{fv}$. Finally, continue with the computation using the perturbed feature vector (i.e., $g^{(L)}(g^{(L-1)}(\ldots(g^{(\tilde{l})}(\sig{fv}))))$), and the computed prediction in the $j$-th dimension should be bounded by~$[L_j, U_j]$. Note that  Definition~\ref{def.perturbation} only requires~$[L_j, U_j]$ to be an over-approximation over the set of all possible predicted values, as the logical implication is not bidirectional. The bound can be computed efficiently with GPU support via approaches such as abstract interpretation with boxed domain (i.e., dataflow analysis~\cite{cousot1977abstract,cheng2017maximum}).

\vspace{1mm}
Given the output bound under $(\tilde{l}, \kappa)$-perturbation, our goal is to define a loss function that computes the overflow of output bounds over the range of tolerant values. 

\vspace{1mm}

\begin{defn}[Symbolic loss]\label{def.symbolic.loss}
	For the $j$-th output of the neural network, for $(\sig{in}, \sig{lb}) \in \mathcal{D}_{train}$, Let~$[L, U]$ be the output bound by feeding the network with~$\sig{in}$ following Definition~\ref{def.perturbation}. Define~$e^{\langle \Delta, \tilde{l}, \kappa \rangle}_{j}([L, U], \sig{lb})$, the \emph{symbolic loss} for the $j$-th output subject to $(\tilde{l}, \kappa)$-perturbation with $\Delta$ tolerance, to be $ \texttt{overflow}([L_j, U_j], [\sig{lb}_j - \Delta_j, \sig{lb}_j + \Delta_j])$. The function  
	$\texttt{overflow}(I_1, I_2)$ equals $\sum^n_{m=1} \texttt{dist}(\frac{1}{2}(L_{\alpha_m} + U_{\alpha_m}), [\sig{lb}_j - \Delta_j, \sig{lb} + \Delta_j])$ where 
	\begin{itemize} 
		\item $[L_{\alpha_1}, U_{\alpha_1}], \ldots, [L_{\alpha_n}, U_{\alpha_n}]$ are maximally disjoint intervals of $I_1 \setminus I_2$. 
		\item  Function $\texttt{dist}(A, [B, C])$ computes the shortest distance between point~$A$ and points in the interval $[B, C]$.
	\end{itemize}

\end{defn}

\vspace{2mm}

The intuition behind the defined symbolic loss is to (i) compute the intervals of the output bound that are outside the tolerance and  subsequently, (ii) consider the loss as the accumulated effort to bring the center of each interval back to the tolerance interval. Figure~\ref{fig:symbolic.loss} illustrates the concept. 

\begin{itemize}
	\item In Figure~\ref{fig:symbolic.loss}-a, as the both the output lower-bound and the upper-bound are contained in the tolerance interval, the loss is set to~$0$.
	\item For Figure~\ref{fig:symbolic.loss}-d, the output bound is $[-2, 5.5]$ while the tolerance interval is $[-1, 4]$. Therefore, there are two maximally disjoint intervals $[-2, -1]$ and $[4, 5.5]$ falling outside the tolerance. The loss is the distance between the center of each interval $[-2, -1]$, $[4, 5.5]$ to the tolerance boundary, which equals $|\frac{1}{2}(5.5 + 4) - 4| + |\frac{1}{2}(-2 + (-1)) - (-1)| = 1.25$. 
	\item Lastly in Figure~\ref{fig:symbolic.loss}-e, the complete interval $[5.5, 9]$ is outside the tolerance interval. The loss is the distance between the center of the interval ($\frac{5.5+9}{2}$) to the boundary, which equals $\frac{5.5+9}{2} -4 = 3.25$. 
	
\end{itemize}

\begin{defn}[Symbolic tolerance loss]\label{def.symbolic.tolerance.error}
	Given a neural network $\mathcal{N}$, define the loss on the training set $\mathcal{D}_{train}$ to be
	$
	E^{(\Delta, \tilde{l},\kappa)}_{train} = \frac{1}{N}\sum_{(\sig{in}, \sig{lb}) \in \mathcal{D}_{train}} (E(\sig{in}, \sig{lb}))
	$, where $E(\sig{in}, \sig{lb}) = \sum^{d_{L}}_{j=1} (e^{(\Delta, \tilde{l}, \kappa)}_{j}([L, U], \sig{lb}))^{2}$, $N = |\mathcal{D}_{train}|$, and $[L, U]$ is computed using Definition~\ref{def.perturbation}.
\end{defn}
\vspace{1mm}

\vspace{2mm}
\noindent \textbf{(Implementing the loss function with GPU support)} The following result states that computing the symbolic loss on the $j$-th output can be done very effectively by averaging the interval loss for $L_j$ and $U_j$, thereby further utilizing the result from Lemma~\ref{lemma.gpu.primitive} for efficient computation via GPU support. 

\vspace{2mm}
\begin{lemma}[Computing symbolic loss by taking end-points]\label{lemma.symbolic.loss}
	For the $j$-th output of the neural network, for $(\sig{in}, \sig{lb}) \in \mathcal{D}_{train}$, the \emph{symbolic loss} subject to  $(\tilde{l}, \kappa)$-perturbation with $\Delta$ tolerance has the following property:
	\[
	e^{\langle \Delta, \tilde{l}, \kappa \rangle}_{j}([L, U], \sig{lb}) = \frac{1}{2} (e^{\Delta}_{j}(L_j, \sig{lb}_j) +  e^{\Delta}_{j}(U_j, \sig{lb}_j))
	\]	
	where $[L, U]$ is computed by feeding the network with $\sig{in}$ using Definition~\ref{def.perturbation}.

\end{lemma}

\vspace{1mm}
\begin{proof}(Sketch)
	Here for simplicity, we illustrate in Figure~\ref{fig:symbolic.loss} all~$6$ possible cases concerning the relative position between interval $[L_j, U_j]$ and interval $[\sig{lb}_j - \Delta_j, \sig{lb}_j + \Delta_j]$. For each case,   results of computing $\frac{1}{2} (e^{\Delta}_{j}(L_j, \sig{lb}_j) +  e^{\Delta}_{j}(U_j, \sig{lb}_j))$ are shown directly in Figure~\ref{fig:symbolic.loss}. Readers can easily swap the constants in Figure~\ref{fig:symbolic.loss} to create a formal correctness proof. 
\end{proof}

\begin{figure}
	\includegraphics[width=0.85\columnwidth]{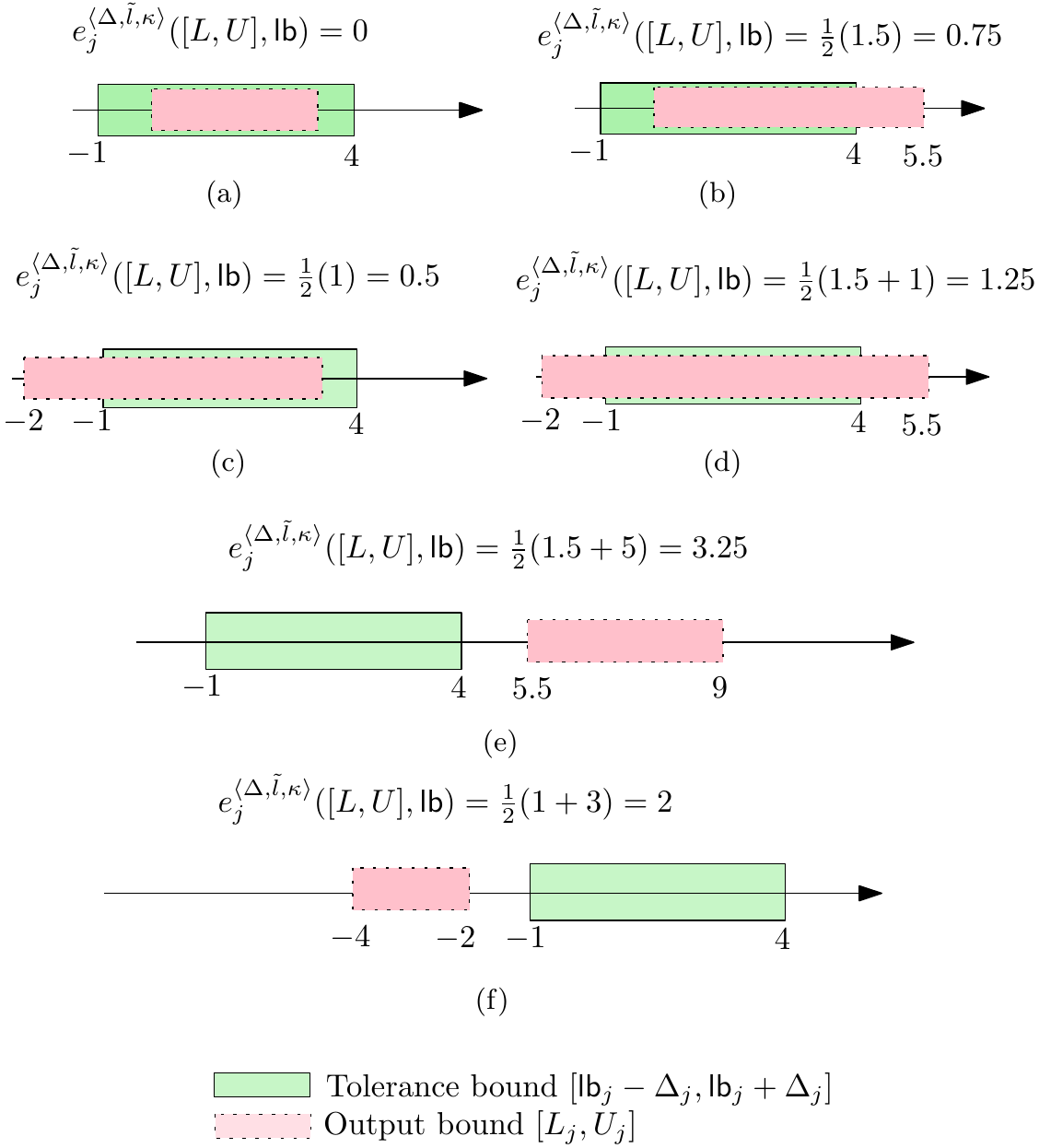}
	\centering
	\caption{Computing symbolic interval loss}
	\label{fig:symbolic.loss}
\end{figure}  

\vspace{1mm}

One immediate observation of Lemma~\ref{lemma.symbolic.loss} is that when the maximum allowed perturbation~$\kappa$ equals~$0$, no feature perturbation appears, and the output bound can be as tight as a single point~$f^{(L)}(\sig{in})$. Therefore, one has $L = U = f^{(L)}(\sig{in})$ and it enables the following simplification. 

\vspace{2mm}
\begin{lemma}\label{lemma.simplification.again}
	When $\kappa = 0$, values computed using symbolic loss can be the same as values computed from interval loss, i.e., $e^{\langle \Delta, \tilde{l}, \kappa \rangle}_{j}([L, U], \sig{lb}) = e^{\Delta}_{j}(f^{(L)}_j(\sig{in}), \sig{lb}_j)$.
\end{lemma}

\vspace{1mm}

Results from Lemma~\ref{lemma.equi.to.mse}, \ref{lemma.symbolic.loss} and \ref{lemma.simplification.again} altogether offer a pragmatic method for training. First, one can train a network using the loss function MSE; based on the chained rule of Lemma~\ref{lemma.equi.to.mse} and Lemma~\ref{lemma.symbolic.loss}, using MSE loss is equivalent to the special case where $\Delta = 0$ and $\kappa = 0$. Subsequently, one can train the network with interval loss; it is equivalent to the special case where~$\kappa=0$. Finally, one enlarges the value of~$\kappa$ towards provably robust training. 

\vspace{2mm}

Finally, we summarize the theoretical guarantee that the new training approach provides. Intuitively, the below lemma states that if there exists an input $\sig{in}'$ (not necessarily contained in the training data) whose feature vector is sufficiently close to the feature vector of an existing input $\sig{in}$, then the output of the network under~$\sig{in}'$ will be close to the output of the network under~$\sig{in}$. 

\vspace{1mm}
\begin{lemma}[Theoretical guarantee on provable training]\label{lemma.provable.robustness} Given a neural network 
	$\mathcal{N}$ and $\mathcal{D}_{train} := \{ (\sig{in}, \sig{lb})\}$ be the training data set, if  $E^{(\Delta, \tilde{l},\kappa)}_{train} = 0$, then for every input $\sig{in'}$, if exists an input training data $(\sig{in},\sig{lb}) \in \mathcal{D}_{train}$ such that $f^{(\tilde{l} - 1)}(\sig{in'}) \in [f^{(\tilde{l} - 1)}(\sig{in}) \oplus  (-\kappa), f^{(\tilde{l} - 1)}(\sig{in}) \oplus + \kappa]$, then $f^{(L)}(\sig{in'}) \in [ \sig{lb} \oplus (-\Delta), \sig{lb} \oplus \Delta]$. 
	
\end{lemma}

\begin{proof}
	When $E^{(\Delta, \tilde{l},\kappa)}_{train} = 0$, from Definition~\ref{def.symbolic.tolerance.error} one knows that for every input data $(\sig{in}, \sig{lb}) \in \mathcal{D}_{train}$, the corresponding $e^{(\Delta, \tilde{l}, \kappa)}_{j}([L, U], \sig{lb})=0$. This implies that the output lower-bound~$L_j$ and upper-bound~$U_j$, computed using Definition~\ref{def.symbolic.loss} with input~$\sig{in}$,  are contained in~$[\sig{lb}_j - \Delta_j, \sig{lb}_j + \Delta_j]$. 
	
	In Definition~\ref{def.symbolic.loss}, the computation of~$L_j$ and~$U_j$ considers every point in $[f^{(\tilde{l} - 1)}(\sig{in}) \oplus  (-\kappa), f^{(\tilde{l} - 1)}(\sig{in}) \oplus + \kappa]$. Therefore, so long as $f^{(\tilde{l} - 1)}(\sig{in'}) \in [f^{(\tilde{l} - 1)}(\sig{in}) \oplus  (-\kappa), f^{(\tilde{l} - 1)}(\sig{in}) \oplus + \kappa]$, the output of the neural network under~$\sig{in'}$ should be within $[L, U]$, thereby within $[ \sig{lb} \oplus (-\Delta), \sig{lb} \oplus \Delta]$.
\end{proof}

\vspace{2mm}
As a consequence, if one perturbs a data point~$\sig{in}$ in the training set to~$\sig{in}'$, so long as the perturbed input has produced similar high-level  feature vectors at layer~$\tilde{l}-1$, the output under perturbation is provably guaranteed to fall into the tolerance interval.

\section{Experiment}~\label{sec.evaluation}

To understand the proposed concept in a realistic setup, we engineered a direct perception network for identifying the center of the ego lane in $x$-position, by considering a fixed height~($y=500$) in pixel coordinates\footnote{It is possible to train a network to produce multiple affordances. Nevertheless, in the evaluation, our decision to only produce one affordance is to clearly understand the impact of the methodology for robustness.}. For repeatability purposes, we take the publicly available TuSimple dataset for lane detection\footnote{TuSimple data set is available at: \url{https://github.com/TuSimple/tusimple-benchmark/wiki}} and create labels from its associated ground truth.

\subsection{Creating data for experimenting direct perception}

In the TuSimple lane detection dataset,  labels for lane markings contain three parts:
\begin{itemize}
	\item $L_y$ containing a list of~$y$ coordinates that are used to represent a lane.
	\item A list of lanes $\mathcal{L}_1, \mathcal{L}_2\ldots$,  where for each lane $\mathcal{L}_i$, it stores a list of~$x$ coordinates. 
	\item The corresponding image raw file. 
\end{itemize}
Therefore, for the $i$-th lane in an image, its lane markings are  $(L_i[0], L_y[0]), (L_i[1], L_y[1]), (L_i[2], L_y[2])$, and so on. The lanes are mostly ordered from left to right, with some exceptions (e.g., files \texttt{clips/0313-1/21180/20.jpg} and \texttt{clips/0313-2/550/20.jpg}) where one needs to manually reorder the lanes. 

We created a script to automatically generate affordance labels for our experiment: First, fix the height to be $500$, followed by finding two adjacent lanes where (1) the first lane marking is on the left side of the image, and (2) the second lane marking is on the right side. If the script cannot find such two lanes, and the script  just omits the data as it requires manual labelling. Subsequently, we take the average $x$-position of two such lanes to be the center of the ego lane (i.e., the label). Therefore, every output label is an integer between~$0$ and~$720$. See Figure~\ref{fig:center.of.the.lane} for the ground truth of the lane marking and the generated center-of-ego-lane position (small green dot). Furthermore, we duplicate images whose  created labels are far ($\geq 100$ pixels) from the center of the image, to highlight the importance of rare events and to compensate the problem of not having enough labelled data. 

\begin{figure}
	\includegraphics[width=\columnwidth]{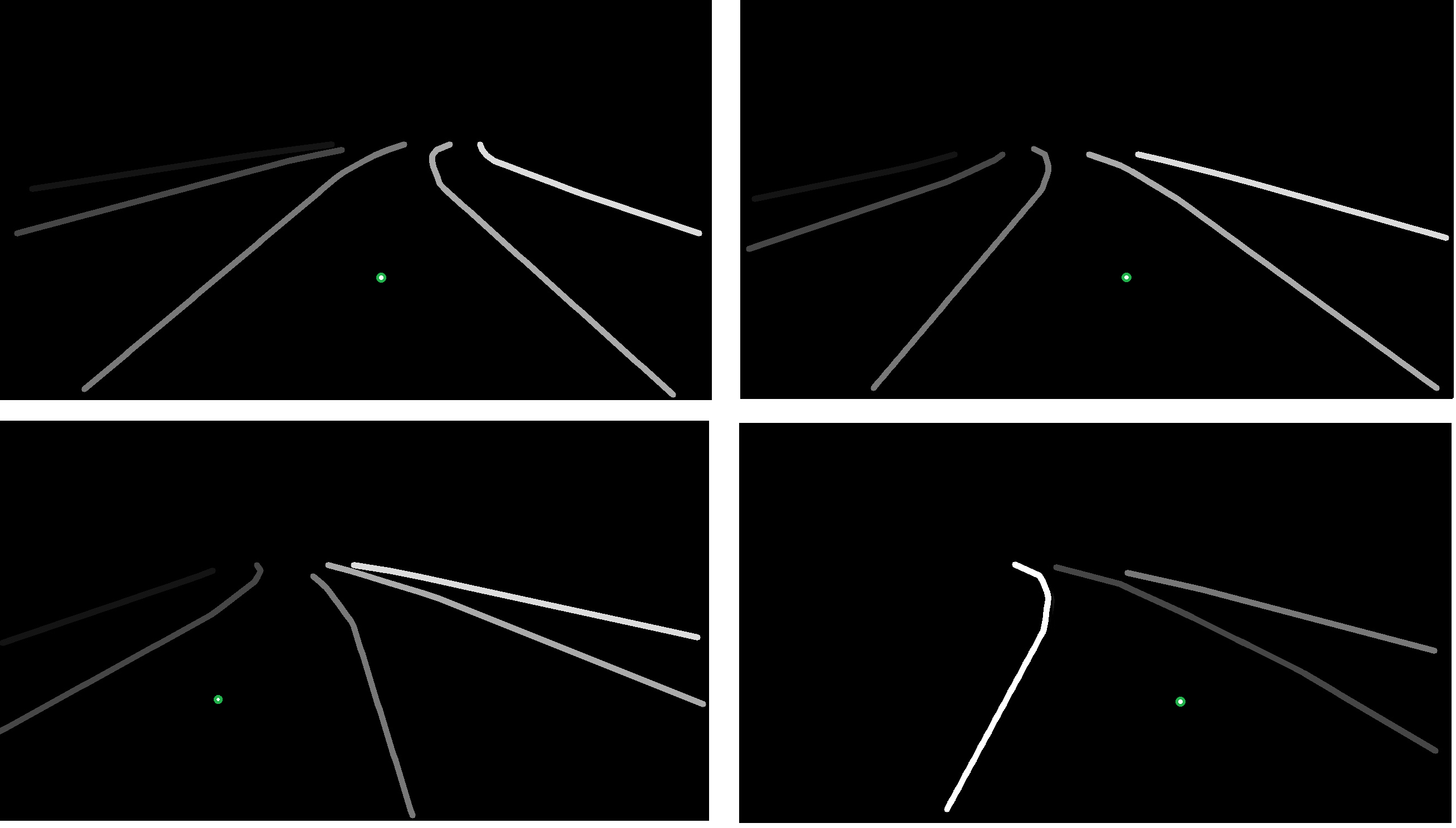}
	\centering
	\caption{Creating labels for the center of the ego-lane (green dot), from labels of lane markings (lines in the images) in the TuSimple data set. }
	\label{fig:center.of.the.lane}
\end{figure}  

For an image in the TuSimple dataset, it has a size of $1280\times720$. We crop the $y$-direction to keep only pixels with indices in range~$[208, 720)$, as the cropped elements are largely sky and cloud. Subsequently, resize the image by~$\frac{1}{4}$ and make it grayscale. This ultimately creates, for each input image, a tensor of dimension $320\times128\times1$ (in TensorFlow, the shape of the tensor equals~$(128,320,1)$). We also perform simple normalization using $t(v) = \frac{2}{255}v - 1$ such that the value~$v$ of each pixel, originally in the range $[0, 255]$, is now in the interval~$[-1, 1]$.

\subsection{Evaluation}

\begin{figure}
	\includegraphics[width=0.35\columnwidth,angle=90]{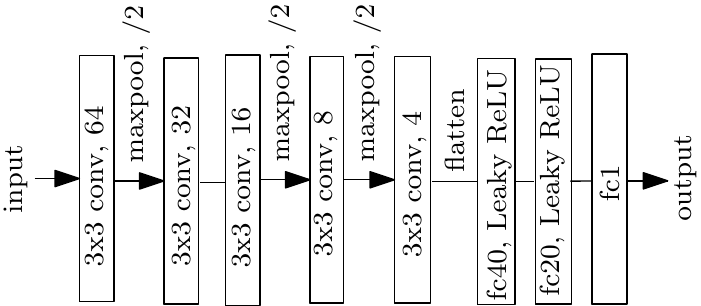}
	\centering
	\caption{The network architecture used in evaluation.}
	\label{fig:network.architecture}
\end{figure}  

In our experiment, we use a network architecture similar to the one shown in  Figure~\ref{fig:network.architecture}. As a baseline, we train~$20$ networks using Xavier weight initialization~\cite{glorot2010understanding}, with each network starting with a unique random seed between~$0$ and~$19$. This is for repeatability purposes (via fixing random seeds) and for eliminating  manual knowledge bias in summarizing our findings (via training many models).  The training uses the Adaptive Moment Estimation (Adam) optimization algorithm~\cite{kingma2014adam} with learning rate~$\alpha=0.01$ for~$20$ epochs and subsequently,~$\alpha=0.001$ for another~$10$ epochs. Finally, we take the best~$6$ performed models and further apply robust training techniques, by using Adam optimization algorithm with $\alpha=0.001$ for yet another~$10$ epochs. In terms of average-case performance, the baseline model and the model further trained with robust loss have similar performance.  

We use the single-step (non-iterative) fast gradient sign method (FGSM)~\cite{szegedy2013intriguing} as the baseline perturbation technique to understand the effect of applying robust training. Precisely, we compare the minimum step size~$\epsilon$ to make an originally perfect prediction (both for the standard network and the network further trained using robust loss) deviate with~$80$ pixels. As we use the single-step method, the parameter~$\epsilon$ is directly related to the intensity of perturbation.
Figure~\ref{fig:bar.chart} shows the overall summary on each baseline model and its further (robustly) trained model. 
In some models (such as model~2), further training does not lead to significant improvement, as for these networks, the minimum~$\epsilon$ values for enabling successful perturbations are largely similar. Nevertheless, for other models such as model~1 or model~6, a huge portion of the images require larger~$\epsilon$ in the robustly trained model, in order to successfully create the adversarial effect. Figure~\ref{fig:symbolic.loss.experiment} details the required~$\epsilon$ value for successful attacks in model~1, where each image is a point in the coordinate plane. One immediately observes that the majority of the points are located at the top-left of the coordinate plane, i.e., one requires larger amount of perturbation for models under robust training to reach the desired effect. 

Although our initial evaluation has hinted promises, it is important to understand that a more systematic analysis, such as evaluating the technique on multiple data sets and a deeper understanding over the parameter space of $\Delta$, $\kappa$, and $\tilde{l}$, is needed to make the technology truly useful. 

\begin{figure}[t]
	\centering
	\begin{tikzpicture}
	\begin{axis}[
	ybar stacked,
	ylabel={Percentage of all images being perturbed},
	legend style={nodes={scale=0.85, transform shape}},
	legend style={at={(0.4,-0.2)},
		anchor=north,legend columns=-1},
	symbolic x coords={model1, model2, model3, model4, 
		model5, model6},
	xtick=data,
	x tick label style={rotate=45,anchor=east},
	]
	\addplot+[ybar] plot coordinates {(model1, 69.1) (model2, 8.6) (model3, 54.1) (model4, 18.6) (model5, 37.1) (model6, 52.5)};
	\addplot+[ybar] plot coordinates  {(model1, 25.2) (model2, 67.6) (model3, 35.9) (model4, 55) (model5, 57.9) (model6, 44.8)};
	\addplot+[ybar] plot coordinates {(model1, 5.6) (model2, 23.8) (model3, 10) (model4, 26.4) (model5, 5) (model6, 2.7)};
	\legend{Robust training has larger $\epsilon$, roughly equal, MSE has larger $\epsilon$}
	\end{axis}
	\end{tikzpicture}    
	\caption{Comparing performance on perturbation, where "roughly equal" refers to the case where the difference of~$\epsilon$ values for two models are smaller than $0.05$.}
	\label{fig:bar.chart}
\end{figure}
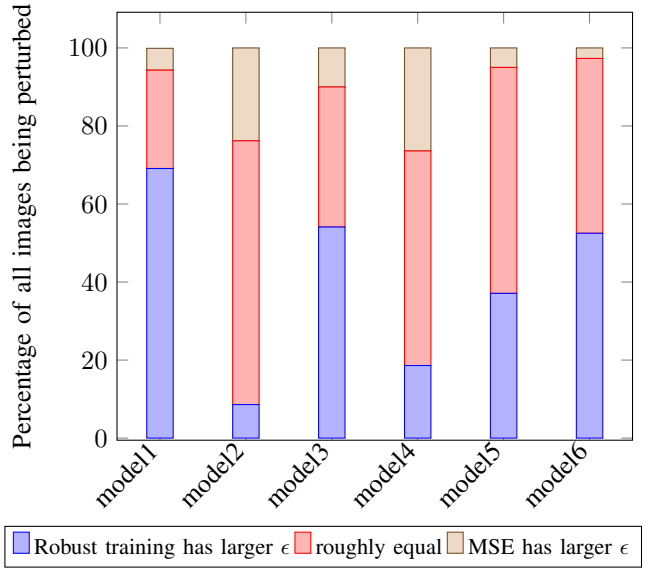

\begin{figure}[t]
	\centering
	\begin{tikzpicture}[thick,scale=0.9, every node/.style={scale=0.9}]
	\begin{axis}[%
	ylabel style={yshift=-0.3cm}, 
	xlabel style={xshift=0cm}, 
	xtick={0.2, 0.4, 0.6, 0.8, 1},
	xticklabels={$0.2$, $0.4$, $0.6$, $0.8$, $\geq 1$},
	xlabel={Required $\epsilon$ for successful FGSM attack (standard training)},
	xmax=1,
	ymax=1,
	ytick={0.2, 0.4, 0.6, 0.8, 1},
	yticklabels={$0.2$, $0.4$, $0.6$, $0.8$, $\geq 1$},
	ylabel={Required $\epsilon$ for successful FGSM attack (robust training)},
	scatter/classes={%
		a={mark=o,draw=black}}]
	\addplot[scatter,only marks,%
	scatter src=explicit symbolic]%
	table[meta=label] {
		x y label
		0.4423828125 0.5537109375 a 
		0.0654296875 0.6611328125 a 
		0.0478515625 0.3740234375 a 
		0.0478515625 0.1357421875 a 
		0.1689453125 0.3037109375 a 
		0.2470703125 0.3623046875 a 
		0.5224609375 0.6865234375 a 
		0.6669921875 0.9931640625 a 
		0.5966796875 0.9033203125 a 
		0.1962890625 0.3037109375 a 
		0.1943359375 0.2841796875 a 
		0.5439453125 0.9443359375 a 
		0.1318359375 0.8447265625 a 
		0.0615234375 0.1669921875 a 
		0.5498046875 0.9990234375 a 
		0.0634765625 0.0810546875 a 
		0.0712890625 0.1669921875 a 
		0.6064453125 0.7412109375 a 
		0.0654296875 0.1416015625 a 
		0.0361328125 0.6376953125 a 
		0.0322265625 0.4892578125 a 
		0.0302734375 0.3740234375 a 
		0.0263671875 0.0341796875 a 
		0.0419921875 0.0498046875 a 
		0.0419921875 0.2021484375 a 
		0.3544921875 0.5009765625 a 
		0.0439453125 0.0498046875 a 
		0.0498046875 0.0517578125 a 
		0.0615234375 0.0615234375 a 
		0.0419921875 0.9990234375 a 
		0.0439453125 0.9990234375 a 
		0.0458984375 0.0439453125 a 
		0.5751953125 0.7333984375 a 
		0.0556640625 0.0595703125 a 
		0.0693359375 0.9814453125 a 
		0.0654296875 0.0712890625 a 
		0.7392578125 0.9990234375 a 
		0.0224609375 0.3525390625 a 
		0.4599609375 0.5556640625 a 
		0.0361328125 0.2099609375 a 
		0.0439453125 0.8193359375 a 
		0.4716796875 0.5908203125 a 
		0.0400390625 0.0439453125 a 
		0.0341796875 0.0419921875 a 
		0.0341796875 0.0419921875 a 
		0.0224609375 0.0263671875 a 
		0.0419921875 0.0458984375 a 
		0.4267578125 0.5302734375 a 
		0.0419921875 0.0458984375 a 
		0.5283203125 0.7001953125 a 
		0.0322265625 0.0322265625 a 
		0.0595703125 0.9482421875 a 
		0.0361328125 0.0439453125 a 
		0.0517578125 0.5849609375 a 
		0.3994140625 0.3759765625 a 
		0.0634765625 0.3818359375 a 
		0.9990234375 0.9990234375 a 
		0.2451171875 0.3427734375 a 
		0.2236328125 0.6923828125 a 
		0.0810546875 0.4560546875 a 
		0.1611328125 0.6083984375 a 
		0.0595703125 0.6552734375 a 
		0.6357421875 0.0634765625 a 
		0.0556640625 0.9990234375 a 
		0.6494140625 0.9189453125 a 
		0.0576171875 0.0712890625 a 
		0.0478515625 0.0517578125 a 
		0.1962890625 0.3037109375 a 
		0.0615234375 0.0771484375 a 
		0.0380859375 0.0361328125 a 
		0.0595703125 0.0634765625 a 
		0.0498046875 0.1513671875 a 
		0.4189453125 0.5556640625 a 
		0.3681640625 0.5048828125 a 
		0.3505859375 0.4677734375 a 
		0.0693359375 0.7099609375 a 
		0.0810546875 0.6318359375 a 
		0.1533203125 0.2783203125 a 
		0.3115234375 0.3896484375 a 
		0.0634765625 0.6142578125 a 
		0.0712890625 0.5869140625 a 
		0.1455078125 0.8955078125 a 
		0.0439453125 0.6455078125 a 
		0.2666015625 0.3955078125 a 
		0.1904296875 0.2607421875 a 
		0.0283203125 0.0283203125 a 
		0.4423828125 0.5615234375 a 
		0.0283203125 0.0302734375 a 
		0.0439453125 0.0458984375 a 
		0.1533203125 0.8798828125 a 
		0.0439453125 0.0498046875 a 
		0.1982421875 0.3154296875 a 
		0.4599609375 0.5673828125 a 
		0.0791015625 0.7529296875 a 
		0.6669921875 0.3115234375 a 
		0.1865234375 0.2509765625 a 
		0.0302734375 0.6884765625 a 
		0.1513671875 0.1845703125 a 
		0.0283203125 0.0302734375 a 
		0.5693359375 0.7958984375 a 
		0.6572265625 0.7880859375 a 
		0.9912109375 0.9990234375 a 
		0.0302734375 0.0361328125 a 
		0.6884765625 0.3916015625 a 
		0.0771484375 0.1748046875 a 
		0.1806640625 0.4404296875 a 
		0.0654296875 0.1337890625 a 
		0.0732421875 0.7099609375 a 
		0.2744140625 0.8701171875 a 
		0.1513671875 0.3193359375 a 
		0.7568359375 0.1572265625 a 
		0.9853515625 0.9990234375 a 
		0.0400390625 0.2177734375 a 
		0.3681640625 0.0732421875 a 
		0.5126953125 0.0732421875 a 
		0.0400390625 0.5927734375 a 
		0.1826171875 0.2451171875 a 
		0.2060546875 0.2529296875 a 
		0.6513671875 0.1630859375 a 
		0.7998046875 0.9990234375 a 
		0.0341796875 0.0380859375 a 
		0.0439453125 0.5927734375 a 
		0.0400390625 0.6083984375 a 
		0.0458984375 0.9384765625 a 
		0.8154296875 0.9326171875 a 
		0.0537109375 0.8759765625 a 
		0.0419921875 0.8916015625 a 
		0.2587890625 0.4775390625 a 
		0.0439453125 0.0537109375 a 
		0.0322265625 0.9990234375 a 
		0.0419921875 0.0478515625 a 
		0.1455078125 0.9794921875 a 
		0.0830078125 0.9990234375 a 
		0.0341796875 0.1650390625 a 
		0.0224609375 0.0224609375 a 
		0.1396484375 0.3330078125 a 
		0.0439453125 0.3388671875 a 
		0.0654296875 0.8349609375 a 
		0.0419921875 0.2607421875 a 
		0.0380859375 0.2041015625 a 
		0.4794921875 0.5654296875 a 
		0.0654296875 0.5810546875 a 
		0.0458984375 0.8388671875 a 
		0.0380859375 0.7919921875 a 
		0.3369140625 0.4501953125 a 
		0.0595703125 0.7490234375 a 
		0.0537109375 0.5185546875 a 
		0.5166015625 0.7958984375 a 
		0.0400390625 0.5947265625 a 
		0.0439453125 0.9189453125 a 
		0.1298828125 0.1533203125 a 
		0.7841796875 0.9560546875 a 
		0.0537109375 0.1806640625 a 
		0.0673828125 0.9130859375 a 
		0.2236328125 0.7783203125 a 
		0.0517578125 0.0615234375 a 
		0.0439453125 0.1435546875 a 
		0.0498046875 0.0517578125 a 
		0.0537109375 0.3388671875 a 
		0.0322265625 0.3212890625 a 
		0.3271484375 0.3857421875 a 
		0.7138671875 0.9541015625 a 
		0.3486328125 0.8369140625 a 
		0.8564453125 0.4345703125 a 
		0.1435546875 0.1689453125 a 
		0.0537109375 0.5341796875 a 
		0.2509765625 0.6083984375 a 
		0.0458984375 0.2021484375 a 
		0.7099609375 0.9990234375 a 
		0.0537109375 0.1982421875 a 
		0.0478515625 0.1396484375 a 
		0.5654296875 0.7021484375 a 
		0.0810546875 0.1630859375 a 
		0.4033203125 0.5458984375 a 
		0.0517578125 0.5654296875 a 
		0.2333984375 0.9990234375 a 
		0.2138671875 0.2763671875 a 
		0.3896484375 0.5693359375 a 
		0.9404296875 0.9990234375 a 
		0.6572265625 0.8740234375 a 
		0.0615234375 0.1806640625 a 
		0.0498046875 0.1962890625 a 
		0.0517578125 0.0673828125 a 
		0.0517578125 0.0634765625 a 
		0.1474609375 0.2177734375 a 
		0.0478515625 0.1728515625 a 
		0.4013671875 0.6220703125 a 
		0.0654296875 0.1904296875 a 
		0.3544921875 0.8779296875 a 
		0.2822265625 0.5927734375 a 
		0.0654296875 0.1669921875 a 
		0.2783203125 0.4697265625 a 
		0.4111328125 0.8798828125 a 
		0.4052734375 0.5146484375 a 
		0.3759765625 0.9951171875 a 
		0.2216796875 0.8076171875 a 
		0.0634765625 0.1455078125 a 
		0.1826171875 0.2060546875 a 
		0.4443359375 0.6005859375 a 
		0.1357421875 0.2099609375 a 
		0.2705078125 0.4482421875 a 
		0.3935546875 0.4814453125 a 
		0.7900390625 0.9365234375 a 
		0.0419921875 0.2021484375 a 
		0.1630859375 0.2236328125 a 
		0.4208984375 0.5380859375 a 
		0.0517578125 0.4873046875 a 
		0.4833984375 0.6220703125 a 
		0.8291015625 0.9755859375 a 
		0.0322265625 0.1767578125 a 
		0.3408203125 0.3974609375 a 
		0.5986328125 0.8388671875 a 
		0.8173828125 0.9990234375 a 
		0.0654296875 0.5849609375 a 
		0.4521484375 0.5830078125 a 
		0.0537109375 0.6845703125 a 
		0.0498046875 0.0517578125 a 
		0.5439453125 0.2392578125 a 
		0.0439453125 0.5419921875 a 
		0.7021484375 0.9208984375 a 
		0.0537109375 0.0595703125 a 
		0.2451171875 0.8544921875 a 
		0.0478515625 0.3642578125 a 
		0.1337890625 0.6904296875 a 
		0.5361328125 0.7373046875 a 
		0.6181640625 0.8291015625 a 
		0.0341796875 0.7529296875 a 
		0.1337890625 0.2216796875 a 
		0.0537109375 0.0673828125 a 
		0.0634765625 0.7333984375 a 
		0.6083984375 0.7626953125 a 
		0.0419921875 0.0419921875 a 
		0.0595703125 0.9990234375 a 
		0.8642578125 0.9990234375 a 
		0.0693359375 0.0830078125 a 
		0.1416015625 0.5244140625 a 
		0.2705078125 0.4677734375 a 
		0.0478515625 0.2099609375 a 
		0.2314453125 0.3154296875 a 
		0.0712890625 0.1806640625 a 
		0.0576171875 0.6298828125 a 
		0.0576171875 0.1533203125 a 
		0.5322265625 0.2490234375 a 
		0.2607421875 0.3154296875 a 
		0.1591796875 0.5830078125 a 
		0.4248046875 0.6474609375 a 
		0.4716796875 0.7314453125 a 
		0.3408203125 0.4873046875 a 
		0.7021484375 0.9658203125 a 
		0.0380859375 0.0400390625 a 
		0.0478515625 0.1806640625 a 
		0.0361328125 0.3779296875 a 
		0.2548828125 0.3427734375 a 
		0.2138671875 0.3408203125 a 
		0.5517578125 0.7197265625 a 
		0.1318359375 0.2255859375 a 
		0.5732421875 0.6865234375 a 
		0.0556640625 0.5283203125 a 
		0.0400390625 0.2001953125 a 
		0.0576171875 0.1708984375 a 
		0.1435546875 0.2236328125 a 
		0.0302734375 0.1923828125 a 
		0.2548828125 0.3486328125 a 
		0.1904296875 0.7724609375 a 
		0.1357421875 0.2197265625 a 
		0.0732421875 0.5146484375 a 
		0.3935546875 0.4775390625 a 
		0.0419921875 0.0439453125 a 
		0.2783203125 0.0576171875 a 
		0.1806640625 0.8076171875 a 
		0.0732421875 0.6689453125 a 
		0.0478515625 0.1845703125 a 
		0.2158203125 0.6103515625 a 
		0.1318359375 0.2568359375 a 
		0.5908203125 0.2255859375 a 
		0.2880859375 0.4912109375 a 
		0.6513671875 0.2900390625 a 
		0.2783203125 0.5517578125 a 
		0.2177734375 0.3056640625 a 
		0.0458984375 0.0537109375 a 
		0.2919921875 0.7958984375 a 
		0.3037109375 0.8154296875 a 
		0.5283203125 0.7197265625 a 
		0.0458984375 0.7841796875 a 
		0.0556640625 0.2451171875 a 
		0.4716796875 0.6083984375 a 
		0.6572265625 0.8408203125 a 
		0.4306640625 0.7158203125 a 
		0.1943359375 0.2685546875 a 
		0.0556640625 0.5732421875 a 
		0.5341796875 0.6826171875 a 
		0.0517578125 0.5458984375 a 
		0.1357421875 0.2099609375 a 
		0.3916015625 0.5654296875 a 
		0.0498046875 0.0498046875 a 
		0.8037109375 0.2861328125 a 
		0.2255859375 0.7431640625 a 
		0.7412109375 0.9775390625 a 
		0.4677734375 0.9990234375 a 
		0.1416015625 0.2197265625 a 
		0.5068359375 0.8525390625 a 
		0.1572265625 0.2470703125 a 
		0.0966796875 0.1923828125 a 
		0.3525390625 0.5185546875 a 
		0.1669921875 0.2666015625 a 
		0.4716796875 0.5693359375 a 
		0.4736328125 0.5986328125 a 
		0.3544921875 0.7294921875 a 
		0.0556640625 0.7509765625 a 
		0.0576171875 0.0732421875 a 
		0.2041015625 0.2763671875 a 
		0.3369140625 0.6904296875 a 
		0.4833984375 0.6025390625 a 
		0.0322265625 0.0439453125 a 
		0.0498046875 0.0634765625 a 
		0.1337890625 0.1865234375 a 
		0.3583984375 0.4541015625 a 
		0.0458984375 0.5087890625 a 
		0.1865234375 0.5068359375 a 
		0.0498046875 0.1630859375 a 
		0.1787109375 0.2607421875 a 
		0.7880859375 0.9990234375 a 
		0.1298828125 0.5419921875 a 
		0.2236328125 0.7744140625 a 
		0.2763671875 0.7236328125 a 
		0.1416015625 0.2392578125 a 
		0.2744140625 0.6806640625 a 
		0.0478515625 0.0595703125 a 
		0.0439453125 0.1650390625 a 
		0.5029296875 0.6201171875 a 
		0.0458984375 0.0556640625 a 
		0.1748046875 0.2373046875 a 
		0.2509765625 0.4345703125 a 
		0.0908203125 0.4130859375 a 
		0.8759765625 0.2392578125 a 
		0.0439453125 0.2001953125 a 
		0.2705078125 0.4267578125 a 
		0.0673828125 0.4833984375 a 
		0.1416015625 0.2216796875 a 
		0.0615234375 0.5693359375 a 
		0.2724609375 0.5224609375 a 
		0.1005859375 0.7939453125 a 
		0.2451171875 0.3251953125 a 
		0.2177734375 0.5126953125 a 
		0.1904296875 0.2255859375 a 
		0.1494140625 0.2568359375 a 
		0.2099609375 0.2626953125 a 
		0.0361328125 0.2294921875 a 
		0.1337890625 0.2822265625 a 
		0.0556640625 0.0751953125 a 
		0.0478515625 0.6162109375 a 
		0.0556640625 0.7158203125 a 
		0.7060546875 0.2685546875 a 
		0.0478515625 0.9560546875 a 
		0.4443359375 0.6064453125 a 
		0.0419921875 0.5576171875 a 
		0.8134765625 0.2841796875 a 
		0.0537109375 0.1708984375 a 
		0.2041015625 0.3134765625 a 
		0.0400390625 0.1826171875 a 
		0.6318359375 0.7958984375 a 
		0.0419921875 0.1376953125 a 
		0.5556640625 0.6572265625 a 
		0.0478515625 0.9111328125 a 
		0.0576171875 0.0693359375 a 
		0.1767578125 0.2353515625 a 
		0.7744140625 0.9306640625 a 
		0.2490234375 0.3115234375 a 
		0.7197265625 0.2861328125 a 
		0.0556640625 0.8447265625 a 
		0.0478515625 0.1474609375 a 
		0.5107421875 0.6123046875 a 
		0.2216796875 0.6513671875 a 
		0.2177734375 0.3154296875 a 
		0.0478515625 0.1494140625 a 
		0.0224609375 0.1376953125 a 
		0.1435546875 0.2021484375 a 
		0.2900390625 0.3759765625 a 
		0.0673828125 0.6572265625 a 
		0.0576171875 0.2705078125 a 
		0.0576171875 0.0693359375 a 
		0.2314453125 0.3134765625 a 
		0.1494140625 0.2099609375 a 
		0.0537109375 0.0595703125 a 
		0.1806640625 0.7177734375 a 
		0.0537109375 0.5556640625 a 
		0.1318359375 0.1630859375 a 
		0.1318359375 0.1669921875 a 
		0.0673828125 0.9677734375 a 
		0.0498046875 0.8212890625 a 
		0.0458984375 0.0576171875 a 
		0.5068359375 0.8232421875 a 
		0.2392578125 0.2978515625 a 
		0.0595703125 0.5908203125 a 
		0.0634765625 0.4970703125 a 
		0.0458984375 0.1982421875 a 
		0.2763671875 0.4970703125 a 
		0.0517578125 0.6005859375 a 
		0.1591796875 0.2705078125 a 
		0.6962890625 0.2353515625 a 
		0.0361328125 0.0341796875 a 
		0.5361328125 0.2099609375 a 
		0.2021484375 0.4169921875 a 
		0.2236328125 0.6083984375 a 
		0.9755859375 0.3427734375 a 
		0.0693359375 0.5908203125 a 
		0.0517578125 0.0576171875 a 
		0.0517578125 0.2353515625 a 
		0.0693359375 0.2490234375 a 
		0.0380859375 0.7978515625 a 
		0.7080078125 0.2275390625 a 
		0.3017578125 0.3974609375 a 
		0.4697265625 0.7705078125 a 
		0.0517578125 0.3583984375 a 
		0.0556640625 0.2255859375 a 
		0.2451171875 0.9990234375 a 
		0.0732421875 0.4716796875 a 
		0.0595703125 0.1630859375 a 
		0.0634765625 0.1689453125 a 
		0.0478515625 0.3798828125 a 
		0.3525390625 0.4541015625 a 
		0.2216796875 0.4013671875 a 
		0.2841796875 0.3408203125 a 
		0.0615234375 0.5615234375 a 
		0.2197265625 0.3056640625 a 
		0.0498046875 0.1708984375 a 
		0.2880859375 0.7724609375 a 
		0.1572265625 0.2001953125 a 
		0.9990234375 0.3447265625 a 
		0.2236328125 0.4482421875 a 
		0.4228515625 0.5693359375 a 
		0.3642578125 0.1552734375 a 
		0.4287109375 0.5283203125 a 
		0.0556640625 0.3798828125 a 
		0.0615234375 0.1552734375 a 
		0.7275390625 0.9990234375 a 
		0.2041015625 0.2529296875 a 
		0.1416015625 0.2197265625 a 
		0.0556640625 0.1630859375 a 
		0.0283203125 0.5595703125 a 
		0.1337890625 0.6064453125 a 
		0.5498046875 0.7158203125 a 
		0.5439453125 0.8056640625 a 
		0.1865234375 0.2490234375 a 
		0.0419921875 0.5576171875 a 
		0.0478515625 0.0595703125 a 
		0.0380859375 0.0400390625 a 
		0.0419921875 0.0478515625 a 
		0.2744140625 0.6904296875 a 
		0.6513671875 0.7822265625 a 
		0.0458984375 0.0458984375 a 
		0.0439453125 0.0478515625 a 
		0.0341796875 0.1728515625 a 
		0.4951171875 0.6181640625 a 
		0.0419921875 0.0498046875 a 
		0.0302734375 0.6630859375 a 
		0.3740234375 0.4658203125 a 
		0.0439453125 0.9072265625 a 
		0.0458984375 0.0498046875 a 
		0.0400390625 0.0458984375 a 
		0.0341796875 0.7529296875 a 
		0.4189453125 0.5830078125 a 
		0.0361328125 0.0400390625 a 
		0.0478515625 0.5126953125 a 
		0.6220703125 0.0341796875 a 
		0.0283203125 0.0302734375 a 
		0.0498046875 0.0537109375 a 
		0.0400390625 0.3701171875 a 
		0.8154296875 0.0068359375 a 
		0.0439453125 0.1572265625 a 
		0.0263671875 0.0302734375 a 
		0.1845703125 0.2626953125 a 
		0.0400390625 0.1455078125 a 
		0.0380859375 0.0458984375 a 
		0.0830078125 0.4814453125 a 
		0.0322265625 0.0361328125 a 
		0.0341796875 0.3583984375 a 
		0.0693359375 0.1435546875 a 
		0.3466796875 0.4404296875 a 
		0.3974609375 0.5322265625 a 
		0.0244140625 0.0283203125 a 
		0.2060546875 0.2822265625 a 
		0.4111328125 0.5029296875 a 
		0.0283203125 0.0302734375 a 
		0.0576171875 0.1513671875 a 
		0.4033203125 0.4560546875 a 
		0.0400390625 0.0419921875 a 
		0.0400390625 0.0400390625 a 
		0.0380859375 0.0458984375 a 
		0.4150390625 0.5498046875 a 
		0.3544921875 0.5458984375 a 
		0.8232421875 0.2529296875 a 
		0.0341796875 0.0400390625 a 
		0.0361328125 0.7314453125 a 
		0.0322265625 0.5009765625 a 
		0.5322265625 0.5751953125 a 
		0.0537109375 0.0400390625 a 
		0.0380859375 0.0498046875 a 
		0.1962890625 0.2412109375 a 
		0.5498046875 0.6318359375 a 
		0.0380859375 0.7236328125 a 
		0.4814453125 0.5791015625 a 
		0.0341796875 0.0400390625 a 
		0.0361328125 0.0361328125 a 
		0.7470703125 0.2509765625 a 
		0.0322265625 0.0322265625 a 
		0.0283203125 0.2041015625 a 
		0.0322265625 0.1611328125 a 
		0.4208984375 0.5244140625 a 
		0.1513671875 0.6552734375 a 
		0.0361328125 0.2646484375 a 
		0.5087890625 0.6083984375 a 
		0.4404296875 0.0439453125 a 
		0.0498046875 0.0537109375 a 
		0.2861328125 0.3388671875 a 
		0.0400390625 0.0400390625 a 
		0.5908203125 0.1669921875 a 
		0.4560546875 0.6513671875 a 
		0.1630859375 0.2470703125 a 
		0.1826171875 0.9990234375 a 
		0.4052734375 0.5009765625 a 
		0.0654296875 0.1826171875 a 
		0.3408203125 0.1435546875 a 
		0.5478515625 0.6611328125 a 
		0.5751953125 0.7392578125 a 
		0.4248046875 0.6279296875 a 
		0.3779296875 0.4853515625 a 
		0.0419921875 0.3525390625 a 
		0.0537109375 0.8779296875 a 
		0.5185546875 0.7138671875 a 
		0.1494140625 0.5458984375 a 
		0.3740234375 0.4931640625 a 
		0.0361328125 0.0419921875 a 
		0.0400390625 0.0458984375 a 
		0.8291015625 0.9990234375 a 
		0.9150390625 0.4599609375 a 
		0.0400390625 0.0400390625 a 
		0.6611328125 0.7880859375 a 
		0.4150390625 0.1806640625 a 
		0.0380859375 0.0361328125 a 
		0.0322265625 0.3525390625 a 
		0.0458984375 0.0517578125 a 
		0.3916015625 0.4951171875 a 
		0.0283203125 0.0322265625 a 
		0.0380859375 0.0400390625 a 
		0.0361328125 0.0380859375 a 
		0.0283203125 0.0322265625 a 
		0.0615234375 0.4697265625 a 
		0.0419921875 0.0458984375 a 
		0.0341796875 0.7587890625 a 
		0.0322265625 0.0341796875 a 
		0.0380859375 0.0419921875 a 
		0.3447265625 0.4521484375 a 
		0.1650390625 0.7978515625 a 
		0.3037109375 0.3662109375 a 
		0.0517578125 0.5048828125 a 
		0.0439453125 0.5244140625 a 
		0.0439453125 0.0537109375 a 
		0.0400390625 0.0419921875 a 
		0.0458984375 0.5263671875 a 
		0.0400390625 0.0419921875 a 
		0.4033203125 0.4677734375 a 
		0.4501953125 0.5849609375 a 
		0.6787109375 0.2021484375 a 
		0.0458984375 0.0458984375 a 
		0.0302734375 0.4033203125 a 
		0.0400390625 0.4287109375 a 
		0.5224609375 0.6455078125 a 
		0.0478515625 0.1513671875 a 
		0.0419921875 0.0517578125 a 
		0.0400390625 0.5732421875 a 
		0.0244140625 0.5673828125 a 
		0.0361328125 0.0419921875 a 
		0.0419921875 0.0439453125 a 
		0.3134765625 0.3759765625 a 
		0.3017578125 0.3759765625 a 
		0.4248046875 0.5947265625 a 
		0.0478515625 0.0517578125 a 
		0.0458984375 0.4833984375 a 
		0.0400390625 0.0419921875 a 
		0.0517578125 0.0498046875 a 
		0.2529296875 0.9990234375 a 
		0.0380859375 0.4150390625 a 
		0.0439453125 0.9248046875 a 
		0.0341796875 0.3466796875 a 
		0.0341796875 0.6845703125 a 
		0.6083984375 0.7041015625 a 
		0.0361328125 0.3642578125 a 
		0.0458984375 0.5771484375 a 
		0.0380859375 0.0419921875 a 
		0.6728515625 0.9833984375 a 
		0.0400390625 0.4287109375 a 
		0.0283203125 0.4638671875 a 
		0.1376953125 0.1962890625 a 
		0.5673828125 0.6923828125 a 
		0.0302734375 0.0302734375 a 
		0.0400390625 0.0419921875 a 
		0.0419921875 0.0458984375 a 
		0.0576171875 0.1513671875 a 
		0.4267578125 0.4970703125 a 
		0.4404296875 0.5263671875 a 
		0.0361328125 0.0380859375 a 
		0.0380859375 0.0419921875 a 
		0.0341796875 0.0361328125 a 
		0.3330078125 0.5419921875 a 
		0.7412109375 0.9990234375 a 
		0.0458984375 0.6396484375 a 
		0.3779296875 0.5634765625 a 
		0.0361328125 0.0341796875 a 
		0.0458984375 0.7705078125 a 
		0.5166015625 0.6611328125 a 
		0.0439453125 0.4873046875 a 
		0.0322265625 0.0341796875 a 
		0.0341796875 0.0361328125 a 
		0.1533203125 0.9990234375 a 
		0.3759765625 0.5185546875 a 
		0.0458984375 0.0458984375 a 
		0.0341796875 0.3759765625 a 
		0.0400390625 0.9833984375 a 
		0.8076171875 0.1962890625 a 
		0.0498046875 0.0498046875 a 
		0.0283203125 0.0302734375 a 
		0.4755859375 0.1650390625 a 
		0.3916015625 0.5029296875 a 
		0.7841796875 0.9990234375 a 
		0.0439453125 0.1748046875 a 
		0.3564453125 0.4521484375 a 
		0.1630859375 0.7939453125 a 
		0.0322265625 0.3779296875 a 
		0.4013671875 0.5400390625 a 
		0.3154296875 0.4033203125 a 
		0.0341796875 0.0361328125 a 
		0.0361328125 0.6611328125 a 
		0.0322265625 0.0380859375 a 
		0.0537109375 0.0615234375 a 
		0.3232421875 0.3857421875 a 
		0.3896484375 0.4658203125 a 
		0.0478515625 0.1708984375 a 
		0.0244140625 0.2939453125 a 
		0.4345703125 0.4833984375 a 
		0.0439453125 0.0458984375 a 
		0.0419921875 0.0458984375 a 
		0.0341796875 0.0361328125 a 
		0.0302734375 0.1923828125 a 
		0.0439453125 0.0498046875 a 
		0.0458984375 0.0458984375 a 
		0.4814453125 0.6240234375 a 
		0.0458984375 0.1474609375 a 
		0.0458984375 0.0517578125 a 
		0.0341796875 0.0361328125 a 
		0.4794921875 0.5263671875 a 
		0.5751953125 0.6806640625 a 
		0.4208984375 0.6474609375 a 
		0.4912109375 0.5400390625 a 
		0.0322265625 0.3583984375 a 
		0.0517578125 0.1826171875 a 
		0.7587890625 0.9072265625 a 
		0.7490234375 0.9287109375 a 
		0.0400390625 0.4345703125 a 
		0.0380859375 0.3056640625 a 
		0.0478515625 0.3583984375 a 
		0.5322265625 0.7802734375 a 
		0.4775390625 0.6826171875 a 
		0.0400390625 0.0458984375 a 
		0.4130859375 0.5986328125 a 
		0.4794921875 0.6318359375 a 
		0.0478515625 0.6220703125 a 
		0.9111328125 0.0048828125 a 
		0.0283203125 0.9990234375 a 
		0.0400390625 0.9990234375 a 
		0.2353515625 0.9697265625 a 
		0.0400390625 0.0458984375 a 
		0.0400390625 0.1806640625 a 
		0.0419921875 0.7822265625 a 
		0.4208984375 0.4970703125 a 
		0.0400390625 0.0439453125 a 
		0.0302734375 0.0341796875 a 
		0.0341796875 0.7353515625 a 
		0.6748046875 0.7509765625 a 
		0.0517578125 0.0478515625 a 
		0.0322265625 0.2919921875 a 
		0.0419921875 0.6962890625 a 
		0.0400390625 0.0458984375 a 
		0.0458984375 0.4912109375 a 
		0.5810546875 0.9072265625 a 
		0.0302734375 0.1845703125 a 
		0.0302734375 0.0322265625 a 
		0.0439453125 0.5517578125 a 
		0.0537109375 0.6337890625 a 
		0.0439453125 0.0498046875 a 
		0.5048828125 0.0224609375 a 
		0.3759765625 0.4755859375 a 
		0.0322265625 0.7392578125 a 
		0.0517578125 0.5146484375 a 
		0.0400390625 0.4775390625 a 
		0.4619140625 0.6279296875 a 
		0.0478515625 0.0537109375 a 
		0.0439453125 0.5126953125 a 
		0.0458984375 0.0458984375 a 
		0.0556640625 0.1337890625 a 
		0.0458984375 0.3974609375 a 
		0.1455078125 0.7119140625 a 
		0.0361328125 0.4931640625 a 
		0.0322265625 0.0341796875 a 
		0.6533203125 0.7333984375 a 
		0.0361328125 0.0341796875 a 
		0.0400390625 0.3134765625 a 
		0.5380859375 0.7900390625 a 
		0.4130859375 0.5322265625 a 
		0.3505859375 0.0537109375 a 
		0.8544921875 0.9990234375 a 
		0.0458984375 0.1513671875 a 
		0.0205078125 0.2587890625 a 
		0.4951171875 0.5361328125 a 
		0.0478515625 0.0556640625 a 
		0.0341796875 0.3544921875 a 
		0.5185546875 0.6123046875 a 
		0.6884765625 0.0361328125 a 
		0.4052734375 0.4912109375 a 
		0.0439453125 0.0498046875 a 
		0.0380859375 0.3291015625 a 
		0.4130859375 0.5556640625 a 
		0.0380859375 0.0439453125 a 
		0.8466796875 0.2919921875 a 
		0.0341796875 0.0419921875 a 
		0.0458984375 0.0478515625 a 
		0.0361328125 0.7841796875 a 
		0.0283203125 0.0283203125 a 
		0.0263671875 0.0302734375 a 
		0.0263671875 0.0302734375 a 
		0.0361328125 0.0361328125 a 
		0.0380859375 0.0400390625 a 
		0.5146484375 0.6748046875 a 
		0.0322265625 0.5830078125 a 
		0.0400390625 0.4501953125 a 
		0.5224609375 0.6279296875 a 
		0.1513671875 0.2353515625 a 
		0.0439453125 0.6181640625 a 
		0.4833984375 0.2880859375 a 
		0.3955078125 0.5341796875 a 
		0.0302734375 0.0380859375 a 
		0.6044921875 0.7646484375 a 
		0.1318359375 0.2177734375 a 
		0.3369140625 0.3857421875 a 
		0.0478515625 0.1435546875 a 
		0.5810546875 0.0458984375 a 
		0.0576171875 0.0654296875 a 
		0.0322265625 0.5556640625 a 
		0.0498046875 0.0537109375 a 
		0.0361328125 0.0322265625 a 
		0.0439453125 0.5556640625 a 
		0.0322265625 0.5556640625 a 
		0.0458984375 0.0439453125 a 
		0.0322265625 0.4169921875 a 
		0.0419921875 0.0556640625 a 
		0.0341796875 0.8408203125 a 
		0.0615234375 0.7001953125 a 
		0.0361328125 0.1826171875 a 
		0.5478515625 0.6728515625 a 
		0.3740234375 0.0283203125 a 
		0.3818359375 0.4912109375 a 
		0.0537109375 0.1728515625 a 
		0.0517578125 0.0673828125 a 
		0.0439453125 0.0478515625 a 
		0.4462890625 0.4794921875 a 
		0.5263671875 0.6689453125 a 
		0.0478515625 0.1455078125 a 
		0.5087890625 0.0595703125 a 
		0.0361328125 0.2822265625 a 
		0.0380859375 0.5126953125 a 
		0.5439453125 0.6962890625 a 
		0.1826171875 0.3037109375 a 
		0.0419921875 0.2353515625 a 
		0.0322265625 0.1923828125 a 
		0.0439453125 0.0439453125 a 
		0.0478515625 0.0595703125 a 
		0.4013671875 0.4716796875 a 
		0.0419921875 0.1455078125 a 
		0.0361328125 0.0478515625 a 
		0.0380859375 0.0400390625 a 
		0.0361328125 0.0439453125 a 
		0.0380859375 0.0380859375 a 
		0.3896484375 0.4501953125 a 
		0.0478515625 0.6162109375 a 
		0.0478515625 0.0537109375 a 
		0.0380859375 0.6064453125 a 
		0.0439453125 0.1455078125 a 
		0.4970703125 0.2060546875 a 
		0.0322265625 0.0341796875 a 
		0.0478515625 0.0498046875 a 
		0.0556640625 0.0595703125 a 
		0.0263671875 0.0283203125 a 
		0.0517578125 0.1533203125 a 
		0.0458984375 0.7041015625 a 
		0.3505859375 0.4091796875 a 
		0.0380859375 0.0439453125 a 
		0.0322265625 0.1611328125 a 
		0.0400390625 0.0380859375 a 
		0.0361328125 0.0380859375 a 
		0.0400390625 0.0419921875 a 
		0.0478515625 0.0615234375 a 
		0.0419921875 0.5009765625 a 
		0.0419921875 0.0478515625 a 
		0.0302734375 0.0322265625 a 
		0.0263671875 0.0283203125 a 
		0.0458984375 0.5009765625 a 
		0.0361328125 0.5224609375 a 
		0.0400390625 0.6181640625 a 
		0.0224609375 0.0263671875 a 
		0.0439453125 0.0458984375 a 
		0.0478515625 0.0576171875 a 
		0.0478515625 0.0537109375 a 
		0.0341796875 0.0341796875 a 
		0.6943359375 0.2548828125 a 
		0.5263671875 0.5908203125 a 
		0.0439453125 0.0419921875 a 
		0.0478515625 0.3974609375 a 
		0.7060546875 0.9990234375 a 
		0.0380859375 0.0400390625 a 
		0.0478515625 0.7705078125 a 
		0.0302734375 0.2041015625 a 
		0.0341796875 0.0361328125 a 
		0.0537109375 0.1513671875 a 
		0.4873046875 0.5654296875 a 
		0.4072265625 0.5576171875 a 
		0.0478515625 0.1474609375 a 
		0.4521484375 0.5986328125 a 
		0.3564453125 0.3818359375 a 
		0.5712890625 0.6435546875 a 
		0.1748046875 0.7373046875 a 
		0.4658203125 0.5869140625 a 
		0.0244140625 0.4853515625 a 
		0.6044921875 0.7666015625 a 
		0.0498046875 0.1884765625 a 
		0.2470703125 0.3408203125 a 
		0.0361328125 0.0380859375 a 
		0.0322265625 0.0341796875 a 
	};
	\end{axis}
	\end{tikzpicture}
	\caption{Comparing models trained using standard (MSE) and symbolic loss approaches (for applying $\Delta = 10$ and $\kappa = 0.01$ on the layer \texttt{fc40} of Figure~\ref{fig:network.architecture}). }
	\label{fig:symbolic.loss.experiment}
\end{figure}
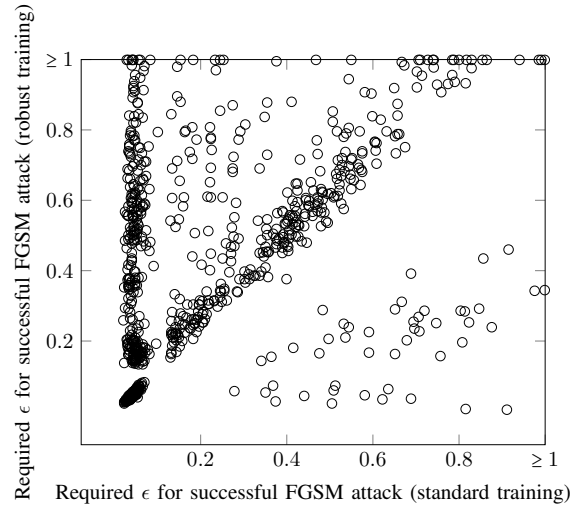

\section{Related Work}~\label{sec.related}

For engineering robust neural networks, the concept in this paper is highly related to the work of provably robust training~\cite{kolter2017provable,sinha2017certifiable,wang2018mixtrain,raghunathan2018certified,wong2018scaling,tsuzuku2018lipschitz,salman2019provably}. Compared to existing work, the key difference lies in two aspects: (i) Existing work focuses on classification, while we focus on output regression for learning affordances. This is made possible by considering an afore-specified tolerance interval. (ii) Existing work perturbs directly on the individual input channels (which is a special case of ours by setting $\tilde{l}=1$), while our definition of layer index~$\tilde{l}$ generally applies to close-to-output layers. This implies a feature-level perturbation rather than perturbation on individual bits (which makes existing methods hard to characterize global transformation such as image distortion). It also increases size of the network that can be trained: As symbolic bound propagation starts at the $\tilde{l}$-th layer, the execution time and the precision for bound propagation is indifferent to the depth of convolution layers before feature layers. In this work, we only use the boxed domain to compute the over-approximation, following the work of~\cite{cheng2017maximum}. One can also use more precise abstract-interpretation domains such as Zonotope~\cite{gehr2018ai2} with a price of increasing the training time, but pragmatically (due to the loss never approaching~$0$) the imprecision of boxed domain can be compensated by simply using smaller~$\kappa$ values.

The concept of provable robustness training by definition only establishes a proof on every data point used in training. Therefore, even for perfectly trained neural networks, it is still insufficient to argue quality assurance on data points in the Operating Design Domain (ODD) that are distant from the training data. It may be complemented by recent research attempts in testing, formal verification, or the introduction of systematic processes for dependable machine learning. Testing techniques, as demonstrated in recent research trend of finding adversarial examples (due to excessive results in this direction, we refer readers to a recent survey paper~\cite{akhtar2018threat} for existing approaches), largely use gradient-based search techniques to find small perturbation around original inputs that makes the output of neural network behave erroneously. It focuses on finding counter-examples rather than providing a guaranteed proof for the absence of undesired behavior. Some recently proposed test metrics aim to provide a pragmatic argument of completeness of testing~\cite{pei2017deepxplore,sun2018testing,sun2018concolic,DBLP:conf/atva/ChengHY18} mimicking the coverage criteria (e.g., line coverage) appeared in classical software testing techniques. Formal verification~\cite{pulina2010abstraction,DBLP:conf/cav/KatzBDJK17,DBLP:conf/cav/HuangKWW17,cheng2017maximum,ehlers2017formal,lomuscio2017approach,narodytska2018verifying,gehr2018ai2,dutta2018output,bunel2018unified,DBLP:conf/ijcai/RuanHK18,wang2018formal,weng2018towards,yang2019analyzing,huang2019reachnn,dvijotham2018dual} views the neural network as a mathematical object and performs symbolic analysis either via SMT~\cite{DBLP:conf/cav/KatzBDJK17,ehlers2017formal}, abstract interpretation~\cite{pulina2010abstraction,gehr2018ai2,DBLP:conf/ijcai/RuanHK18,yang2019analyzing}, MILP~\cite{cheng2017maximum,lomuscio2017approach}, or specialized search techniques~\cite{DBLP:conf/cav/HuangKWW17,dutta2018output,bunel2018unified,narodytska2018verifying,wang2018formal,weng2018towards}. Overall, formal verification offers strong promise on the absence undesired behaviors, but the scalability to very deep networks with inputs from high-dimensional pixel images remains limited.

\section{Concluding Remarks}~\label{sec.concluding.remarks}

In this paper, we considered the problem of engineering robust networks for direct perception, where regression rather than classification is used for output. We defined a loss function by integrating the practically-driven \emph{tolerance} concept, thereby guiding the training process with the goal of bringing predictions back to the tolerance interval rather than a stricter form of being close to labels. Extending the tolerance concept by integrating perturbation, we created conditions where one can ascertain robustness with provably guarantees. The proposed loss functions are proven to be generalizations of the MSE loss function commonly used in standard training approaches. 

The inability to create a systematic approach for engineering robust neural networks for perception systems is one of the most critical barriers towards safe automated driving. We believe that the approach suggested in this paper, in particular the tolerance-driven approach for machine learning, offers an initial step towards a rigorous and contract-driven methodology for the use of machine learning in the automotive domain. This is largely due to the alignment between the tolerance as regulated in the contract (specification) and the corresponding tolerance-integrated loss function (architecture design). Although the work is motivated by concrete problems in direct perception, the underlying technique can be applicable to other applications in automated driving where regression is used. 

For future work, we plan to apply similar concepts to other domains such as medical diagnosis, as well as developing analogous concepts for recurrent neural networks. Within the automotive domain, we also plan to bring our proposed contract-based approach for machine learning to standardization bodies, with the goal of revising existing autonomous driving safety standards such as ISO~21448\footnote{SOTIF: \url{https://www.iso.org/standard/70939.html}}.

\bibliographystyle{IEEEtran}

\end{document}